\documentclass{article}

\usepackage{microtype}
\usepackage{graphicx}
\usepackage{subfigure}
\usepackage{booktabs} %
\usepackage{multirow}
\usepackage{tcolorbox}
\usepackage{CJKutf8}

\usepackage{hyperref}

\usepackage[accepted]{icml2024}

\usepackage{amsmath}
\usepackage{amssymb}
\usepackage{mathtools}
\usepackage{amsthm}
\usepackage{optidef}
\usepackage{dsfont}
\usepackage{algorithm}
\usepackage{algorithmic}
\usepackage{subfigure}

\usepackage[capitalize,noabbrev]{cleveref}

\theoremstyle{plain}
\newtheorem{theorem}{Theorem}[section]

\theoremstyle{definition}

\theoremstyle{remark}

\usepackage[textsize=tiny]{todonotes}

\newcommand{\etal}{\textit{et al}.}
\newcommand{\ie}{\textit{i}.\textit{e}.}

\newcommand{\etc}{\textit{etc}.}

\icmltitlerunning{ExCP: Extreme LLM Checkpoint Compression via Weight-Momentum Joint Shrinking}

\begin{document}
	\begin{CJK}{UTF8}{gbsn}
	
\twocolumn[
\icmltitle{ExCP: Extreme LLM Checkpoint Compression via Weight-Momentum Joint Shrinking}

\begin{icmlauthorlist}
\icmlauthor{Wenshuo Li}{yyy}
\icmlauthor{Xinghao Chen}{yyy}
\icmlauthor{Han Shu}{yyy,sch}
\icmlauthor{Yehui Tang}{yyy}
\icmlauthor{Yunhe Wang}{yyy}
\end{icmlauthorlist}

\icmlaffiliation{yyy}{Huawei Noah’s Ark Lab}
\icmlaffiliation{sch}{University of Science and Technology of China}

\icmlcorrespondingauthor{Xinghao Chen}{xinghao.chen@huawei.com}
\icmlcorrespondingauthor{Yunhe Wang}{yunhe.wang@huawei.com}

\icmlkeywords{Machine Learning, ICML}

\vskip 0.3in
]

\printAffiliationsAndNotice{} %

\begin{abstract}
Large language models (LLM) have recently attracted significant attention in the field of artificial intelligence.
However, the training process of these models poses significant challenges in terms of computational and storage capacities, thus compressing checkpoints has become an urgent problem.
In this paper, we propose a novel Extreme Checkpoint Compression (ExCP) framework, which significantly reduces the required storage of training checkpoints while achieving nearly lossless performance. 
We first calculate the residuals of adjacent checkpoints to obtain the essential but sparse information for higher compression ratio.
To further excavate the redundancy parameters in checkpoints, we then propose a weight-momentum joint shrinking method to utilize another important information during the model optimization, \ie, momentum.
In particular, we exploit the information of both model and optimizer to discard as many parameters as possible while preserving critical information to ensure optimal performance.
Furthermore, we utilize non-uniform quantization to further compress the storage of checkpoints. 
We extensively evaluate our proposed ExCP framework on several models ranging from 410M to 7B parameters and demonstrate significant storage reduction while maintaining strong performance. For instance, we achieve approximately $70\times$ compression for the Pythia-410M model, with the final performance being as accurate as the original model on various downstream tasks. Codes will be available at https://github.com/Gaffey/ExCP.
\end{abstract}

\section{Introduction}
\label{sec:intro}
Large Language Model (LLM)~\cite{gpt3, llama, wangpangu, palm, team2023gemini} has attracted the attention of the vast majority of academia and industry concentrated on Artificial Intelligence (AI). The current LLM can conduct daily conversations with humans, ask questions and answer questions, help humans extract information from articles and charts, and even complete professional-related tasks such as consultation and programming, which greatly improves the efficiency of human-computer interaction. Thousands of laboratories and companies are involved in the training of the LLMs. Computing power and storage have become key resources in the LLM era. Training an LLM requires up to thousands of GPUs or computing cards like TPUs or Ascends, and it is difficult to keep such a large computing cluster running completely smoothly. At the same time, researchers are also faced with the need to interrupt training at any time to adjust training data and hyperparameters. Sometimes it is even necessary to go back to earlier checkpoints to solve problems introduced during training. Therefore, frequent saving of checkpoints has become a must during the whole training process.

\begin{figure*}[th]
\centering
\vspace{0pt}
\includegraphics[width=6.8in]{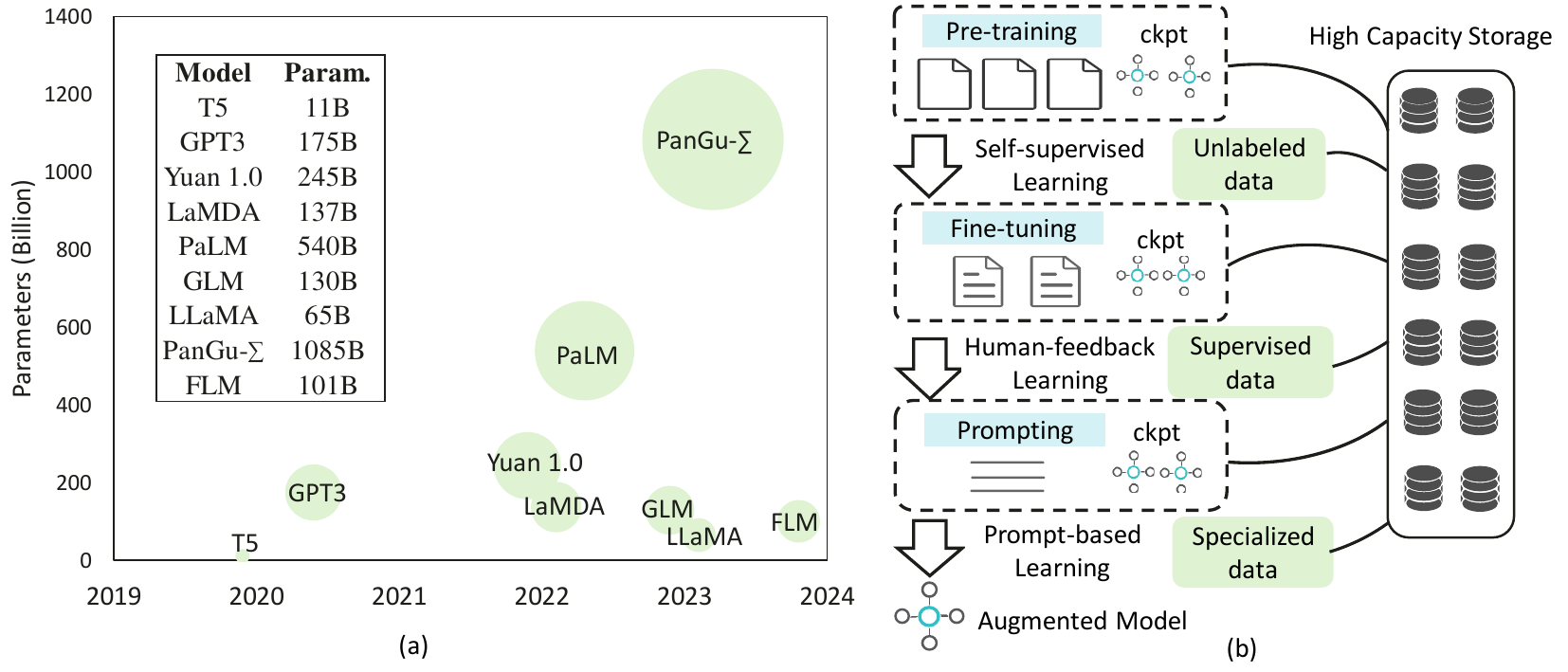}
\vspace{-8mm}
\caption{The number of parameters of some LLMs and the general training process of LLMs.  (a) Parameters of some recent LLMs, most of them contain billions of weights and keep getting larger in trend. (b) The training of LLMs consists of several stages with variety of schemes and data.  A large quantity of checkpoints would be stored in each stage. Considering the magnitude of LLMs’ parameters, extremely high capacity storage is needed for training of LLMs, which could cost tens of millions of dollars.}
\vspace{-5mm}
\label{fig:motivation}
\end{figure*}

Take the open source model Pythia~\cite{biderman2023pythia} as an example, the checkpoint of the largest version Pythia-12B model takes more than 24GB to save. Not to mention the relevant momentum states of the optimizer. Adam optimizer requires twice the storage space of the weight. The training process of Pythia-12B saves 154 checkpoints which requires about 11TB storage, which would cost \$5000 a month on a general cloud server to store these checkpoints. And this is just an entry-level scenario for large company. Conservative estimates suggest that the largest models of the most advanced LLMs, such as the GPT series and Gemini series, has the number of parameters on the order of hundreds to thousands of billions. Some publicly available data is shown Figure~\ref{fig:motivation}. Larger models also require more checkpoints and longer training time. So the total cost of storage for a cutting-edge LLM may grow to tens of millions of dollars.

In view of the above problems, compressing model checkpoints has become a very urgent need. Model compression itself is not a new topic. The model size is compressed to reduce the storage occupied by checkpoints~\cite{han2015deep, hu2020delta, eisenman2022check, chen2020efficient, jin2023design, dynaquant} or compress the calculation amount of the model to improve the model's inference performance~\cite{networkslimming,tang2020scop, dettmers2022llm, xiao2023smoothquant, chen2022mtp, shu2023tinysam,wu2023ppt}. These researches have drawn attentions of researchers in the past ten years. However, previous checkpoints compression work concerns more about the size of weight checkpoints instead of the whole training states, so there is a lack of relevant researches on momentum states compression. In addition, the similarity of adjacent checkpoints should also be considered in the compression pipeline. This feature can improve the pruning ratio instead of simply reducing the final size using some encoding techniques.

In this paper, we propose a checkpoints compression framework that does not rely on training code and information. We calculate the residual value of adjacent checkpoints, apply weight-momentum joint pruning, and then non-uniformly quantize the weights and momentum states to extremely compress the checkpoints. Meanwhile, our residual compression strategy ensures that we can resume the training from compressed checkpoints nearly lossless. Our main contributions are as follows:
\begin{itemize}
    \item We propose a checkpoints compression framework which contains residual calculation, weight-momentum joint pruning and non-uniform quantization. This framework makes full use of the characteristics of checkpoints compression, achieving almost lossless training recovery while achieving a high compression ratio. 
    \item We derive a weight-momentum joint pruning method, and prove the convergence of the optimizer under this pruning method. This is the first work to our knowledge that jointly considers both weights and momentum states pruning.
    \item We conduct experiments on various models and evaluation benchmarks. Our compressed model achieves up to 70$\times$ nearly lossless compression on Pythia-410M model, which could largely reduce the storage of saving checkpoints.
\end{itemize}

\section{Related Work}
\label{sec:related-work}

\subsection{Large Language Model}
Recently, the emergence of large language model and the corresponding strong capabilities in various natural language processing (NLP) applications have drawn widespread attention in research society. The demonstrated powerful abilities by the model scaling have furthermore increased the parameters of large language models. The remarkable work GPT3~\cite{gpt3} shows impressive performance on solving real-world NLP tasks. However, as shown in Table~\ref{table:ckpt_size}, the model contains 175 billion parameters and requires large amount of hardware resources to be trained and stored. A single training checkpoint of GPT3 can reach up to 2.3TB. Following large language models such as PaLM~\cite{palm} and LLaMA~\cite{llama} consume comparable or even more hardware resources. Due to the huge resource consumption and common training failures, the checkpoints of LLMs should be updated and stored frequently, which could occupy much more resources of the communication bandwidth and storage devices. Thus, the exploration of redundancy in LLM checkpoint is meaningful and necessary, which can save the memory consumption in great extent and make the training procedure more efficient and more affordable.

\begin{table}[t]
	\vspace{-5pt}
	\caption{The parameter and checkpoint size of part LLMs. High-capacity storage devices are essential for checkpoints for LLM training process.}
	\centering
	\begin{tabular}{c|c|c}
		\toprule
		Model & Param. & Storage   \\ \midrule
		GPT3~\cite{gpt3} &  175B       &      2.3TB   \\  %
		PaLM~\cite{palm} &  540B       &       $\sim$7TB    \\  %
		LLaMA-70B~\cite{llama} &  75B       &       1.0TB   \\  %
		PanGu-$\pi$ ~\cite{wangpangu} &  7B       &       99GB   \\  \bottomrule
	\end{tabular}
	\label{table:ckpt_size}
	\vspace{-7mm}
\end{table}

\subsection{Compression Method}
\noindent\textbf{Data Compression Methods.}
Compression methods for efficient storage of data have been investigated for long decades.  These previous methods can be categorized into two types, the lossy and the lossless. Lossy compression methods like JPEG~\cite{coding_techniques} and MP3~\cite{mp3} are widely used in the compression of image and video data which does not require precise restoration. Huffman coding~\cite{huffman_coding} is a classic lossless compression method, which statisticizes the frequency of the characters to get an optimized coding length according to different frequency of occurrence. The lossless compression method can be easily applied to the checkpoints of LLMs. However, the generalizability of the data compression method determines that the compression rate would be relatively low when applied to LLM checkpoints. Specialized compression method should be investigated and designed to achieve a higher compression rate for heavy intrinsic redundancy of LLM checkpoint.  

\noindent\textbf{Neural Network Compression Methods.}
The neural network compression methods have been explored by many work due to the increasing model size and computation resources. DeepCompression~\cite{han2015deep} utilizes network pruning, quantization and huffman coding to obtain a compact neural network. Llm~\cite{dettmers2022llm} and Smoothquant~\cite{xiao2023smoothquant} adopt the quantization to compress the large language models. These network compression methods could reduce the quantity or bit-width of parameters in neural networks but are often highly related to the training targets. Thus, these methods cannot be generally applied to compress the checkpoints with various task background. Moreover, re-training or finetuning is often necessary for compression methods like network pruning~\cite{networkslimming} and quantization-aware training~\cite{qat}, which could be extremely computationally expensive especially for large language models with huge training data and huge amount of network parameters.  

\noindent\textbf{Compression methods for checkpoints.}
As the deep neural network model getting larger and the training cost getting more expensive, some research work begin to focus on the compression of checkpoints. LC-Checkpoint~\cite{chen2020efficient} proposes a lossy compression scheme for checkpoint constructions on the assumption of SGD optimizer. Check-N-Run~\cite{eisenman2022check} applies differential and quantization for recommendation models. Delta-DNN ~\cite{hu2020delta} focuses on the storage of floating point numbers and records the differential of two neighboring versions. QD-Compressor~\cite{jin2023design} further develops a layer-based quantization and achieves higher compression ratio. When these methods applied on large-scale models of LLM, undesirable accuracy degradation would occur due to the uniform and constant quantization strategy during training procedure. Recent DynaQuant~\cite{dynaquant} tackles this issue by precisely compressing model parameters based on different contributions to the final result quality with an efficient dynamic quantization configuration and a quantization-aware delta encoding scheme. However, most of previous work focus on the compression of model parameters while ignoring the momentum states of optimizer, which occupy more memory storage and exist more redundancy in LLM checkpoints.

\section{Our Method}
\label{sec:method}

\begin{figure*}[t]
\centering
\vspace{-5pt}
\includegraphics[width=\linewidth]{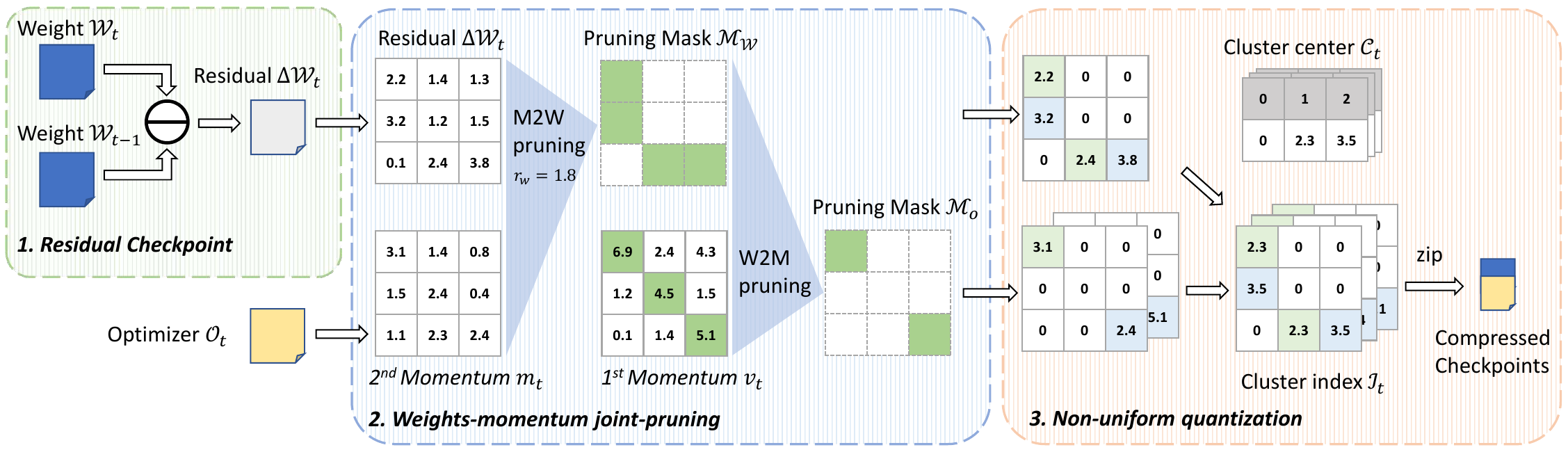}
\vspace{-6mm}
\caption{Framework of our proposed compression process. We calculate the residual $\Delta \mathcal{W}_t$ and apply joint-pruning on $\Delta \mathcal{W}_t$ and $\mathcal{O}_t$. Then we quantize them separately and save the cluster center $\mathcal{C}_t$ and cluster index $\mathcal{I}_t$.}
\vspace{-8pt}
\label{fig:framework}
\end{figure*}

A checkpoint $\mathcal P_t$ of a neural network during the $t^{th}$ training iteration generally contains the model weights $\mathcal W_t$ and parameters $\mathcal O_t$ of the optimizer momentum.
\begin{equation}\label{eq:ckpt}
	\mathcal P_t = \{\mathcal W_t, \mathcal O_t\}.
\end{equation}
Saving checkpoints for $T$ times during training leads to a series of checkpoints $\mathcal P$.
\begin{equation}\label{eq:ckpt_all}
	\mathcal P = \{\mathcal P_1, \mathcal P_2, \cdots, \mathcal P_t\, \cdots, \mathcal P_T\}.
\end{equation}
For the widely used Adam optimizer, the parameters with most significant storage cost are the first-order and second-order moments $v_t$, $m_t$, \ie, 
\begin{equation}\label{eq:ckpt_adam}
	\mathcal O_t = \{v_t, m_t\}.
\end{equation}
Note that some variables such as learning rate and weight decay \etc\ are also stored in the optimizer checkpoint, but can be simply neglected when compared with moments.

In the traditional pruning-related work, researchers only concern about the weights of models, since the main purpose of pruning is reducing the overhead of calculation and storage during the inference stage. However, when we turn to the checkpoint compression during the training process, the pruning of momentum is also important to reduce the total size of training checkpoints. Take the most general optimizer Adam used in LLM training as an example, it saves the first-order and second-order moment of gradients which require double storage of weights. Therefore, we have to take both model weights and optimizer momentum states into consideration for extreme compression of model training checkpoints.

\subsection{Residual Checkpoint}
\label{ssec:dyna-res}
During the $t^{th}$ iteration of training, since we have already stored $t-1$ checkpoints\footnote{We give a detailed description about how to deal with previous checkpoints saved in residual format in Section~\ref{ssec:comp-recon}} during previous training period, it is important to jointly utilize the temporal information of checkpoints to obtain more compact storage. The model weights will be updated upon previous ones according to the gradient during training, thus the difference between adjacent model weights is mostly to be sparse, which is more suitable for compression. In contrast, the momentum states stored in the optimizer checkpoints are the moving average of the first-order and second-order moments, which are only weakly related to the parameters in the previous checkpoint after updating for hundreds to thousands of steps, espectially for the first-order moment whose general $\beta_1=0.9$. So we do not apply residual calculation on optimizer momentum. The residual checkpoint $\Delta \mathcal{P}_t$ is defined as 
\begin{equation}\label{eq:ckpt_res}
	\Delta \mathcal P_t = \{\Delta \mathcal W_t, \mathcal O_t\} = \{\mathcal W_t-\mathcal W_{t-1}, \mathcal O_t\}.
\end{equation}

There is a comparison between the pruning on residual checkpoint and pruning on original checkpoint in Figure~\ref{fig:dist-weight}. We plot the histogram of the original weights, weights after direct pruning and weights after pruning on residual checkpoints. We find that pruning the residual checkpoint has almost no impact on the parameter distribution. This helps us to further prune the parameters.

\begin{figure*}[th]
\centering
\vspace{-5pt}
\includegraphics[width=6in]{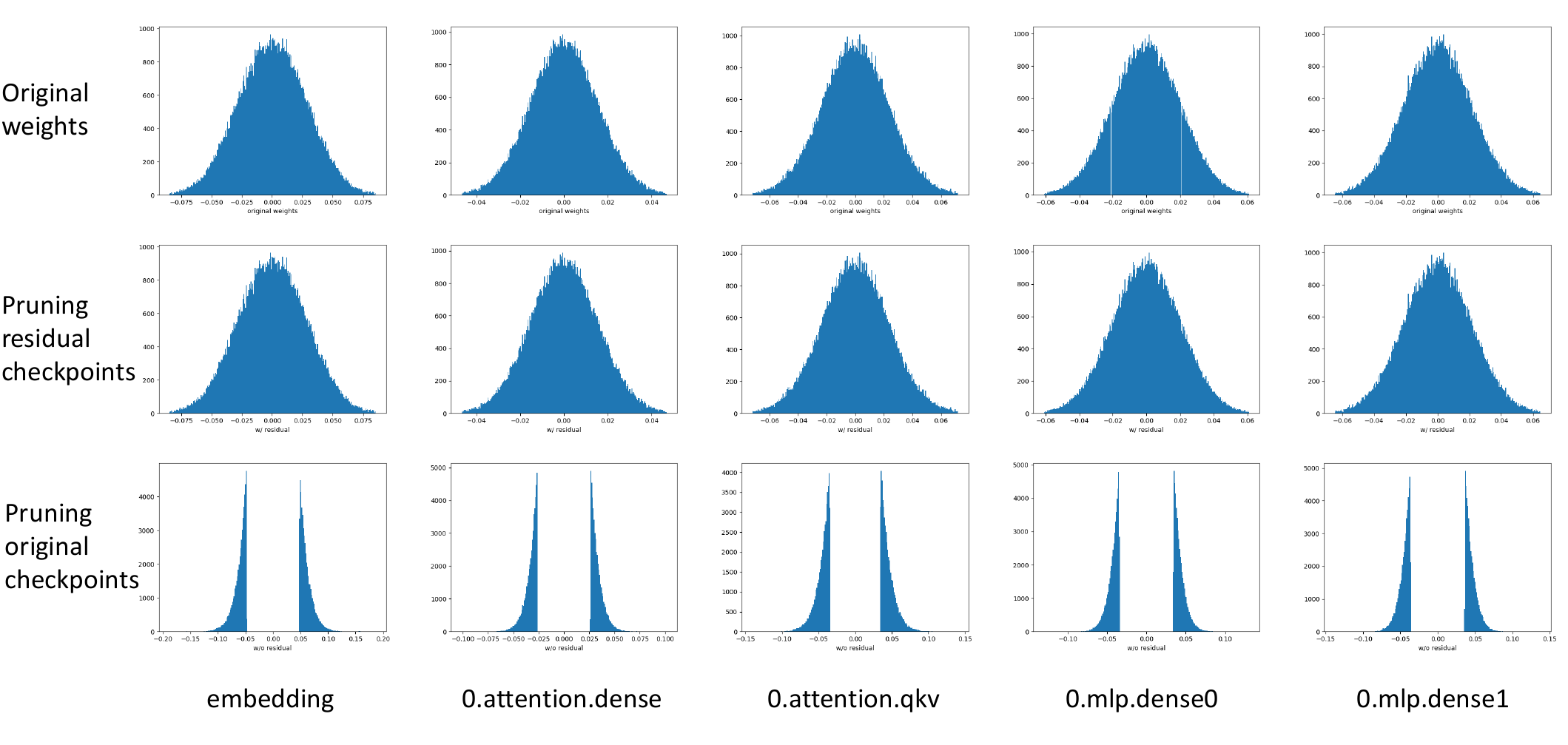}
\vspace{-15pt}
\caption{Weights distribution for original weights, pruning on residual checkpoints and pruning on original weights. We plot the histogram of random 100k non-zero weights of each case for clarity. The range of bins are bounded by (mean - 3 * std, mean + 3 * std) and 256 bins are used.}
\vspace{-10pt}
\label{fig:dist-weight}
\end{figure*}

\subsection{Joint Weight-Momentum Pruning}
\label{ssec:joint-pruning}

Weight pruning is a common way to discard unimportant values while maintaining the performance to the maximum extent. For the checkpoint compression, we need to obtain the corresponding pruning masks for model weights and momentum states, which are denoted as $\mathcal{M}_w$ and $\mathcal{M}_o$, respectively.
Intuitive way for pruning model weights and momentum states is to discard values with some pre-defined metric. However, this separate strategy leads to sub-optimal solution since there are strong relations between model weights and momentum states. Therefore, in this paper we propose a novel joint weight-momentum pruning method that obtains better performance for checkpoint compression.

\textbf{Weight Pruning.} For weights pruning, using the magnitude or the gradients of weights as an indicator is the common practice. There is a little difference between our weight pruning task and the general one. As we introduced in Section~\ref{ssec:dyna-res}, we need to prune the residual values of weights of two adjacent checkpoints instead of the original value of weights. Thus we recommend to use the second-order moment of gradients of weights as an indicator, since they can represent the statistical average of the weight change during training process. We use the indicator to calculate the pruning threshold of each layer and the formula is shown below,
\begin{equation}
\label{eqa:rw}
	r_w = \frac{\alpha}{\sqrt{m_t}} \times \text{median}(\mathcal{W}),\\
	\mathcal{M}_w(i) = \mathds{1}_{w_{t}(i) > r_w}.
\end{equation}
In which $\mathcal{W}$ and $m_t$ represents the weights and the second-order moment, respectively. $\alpha$ is a hyperparameter. After determining the pruning threshold of each layer, we prune the residual of weights to zero by magnitude for each layer.

\textbf{Momentum Pruning.} For momentum pruning, we use the first-order moment as an indicator to determine whether to prune this momentum states or not. We give a brief proof in the next section to explain why we choose it.  Besides, if a specific location of weights is pruned, intuitively it is not important to preserve the corresponding momentum states. We prune the parameters of momentum following the formula below, in which $\beta$ is a hyper-parameter.
\begin{equation}
\label{eqa:ro}
	r_o = \beta \times \text{mean}(v_t),\\
	\mathcal{M}_o(i) = \mathds{1}_{v_{t}(i) > r_o~\text{and}~\mathcal{M}_w(i)=1}.
\end{equation}

\textbf{Convergence Analysis.} Since we prune both the model weights and momentum states during training, it is important to analyze whether the whole training with checkpoint compression still have convergence guarantee.

\begin{theorem}\label{theorem}
	According the convergence analysis in Adam~\cite{kingma2014adam}, assume that the function $f_t$ has bounded gradients, $\left\|\nabla f_t(\theta)\right\|_2 \leq G,\left\|\nabla f_t(\theta)\right\|_{\infty} \leq$ $G_{\infty}$ for all $\theta \in R^d$ and distance between any $\theta_t$ generated by Adam is bounded, $\left\|\theta_n-\theta_m\right\|_2 \leq D$, $\left\|\theta_m-\theta_n\right\|_{\infty} \leq D_{\infty}$ for any $m, n \in\{1, \ldots, T\}$, and $\beta_1, \beta_2 \in[0,1)$ satisfy $\frac{\beta_1^2}{\sqrt{\beta_2}}<1$. Let $\alpha_t=\frac{\alpha}{\sqrt{t}}$ and $\beta_{1, t}=\beta_1 \lambda^{t-1}, \lambda \in(0,1)$. If we prune the moments with a mask $\mathcal{M}_o\in \{0,1\}$ at iteration $\tau$, 
	Adam could also achieves the following guarantee, for all $T \geq 1$.
	\begin{equation}
		\small
		\begin{aligned}\label{bound}
			\tilde R(T) 
			\leq & \frac{D^2}{2 \alpha\left(1-\beta_1\right)} \sum_{i=1}^d \sqrt{T \widehat{v}_{T, i}}\\
			&+\frac{\alpha\left(1+\beta_1\right) G_{\infty}}{\left(1-\beta_1\right) \sqrt{1-\beta_2}(1-\gamma)^2} \sum_{i=1}^d\left\|g_{1: T, i}\right\|_2 \\
			&+\frac{D_{\infty}^2 G_{\infty} \sqrt{1-\beta_2}}{2 \alpha} \sum_{i=1}^d \sum_{t=1}^t \frac{\beta_{1, t}}{\left(1-\beta_{1, t}\right)} \sqrt{t}\\
			& +\frac{D^2}{2 \alpha\left(1-\beta_1\right)} \sum_{i=1}^d (\sqrt{T \widehat{v}_{\tau, i}} - \sqrt{T \widehat{v}_{\tau, i}\mathcal{M}_o})
		\end{aligned}
	\end{equation}
\end{theorem}
Compared with the original convergence analysis of Adam~\cite{kingma2014adam}, the regret bound for our checkpoint compression method has an additional term:
\begin{equation}
	\begin{aligned}
		\Delta\tilde R(T)& = \frac{D^2}{2 \alpha\left(1-\beta_1\right)} \sum_{i=1}^d (\sqrt{T \widehat{v}_{\tau, i}} - \sqrt{T \widehat{v}_{\tau, i}\mathcal{M}_o})\\
		&= \frac{D^2}{2 \alpha\left(1-\beta_1\right)} \sum_{i=1}^d (\sqrt{T \widehat{v}_{\tau, i}(1-\mathcal{M}_o)} ). 
	\end{aligned}
\end{equation}
Since we only prune the values that $v$ is relatively small, thus the regret bound is quite close to that of original training process.

Similar to the original optimization process of Adam, the average regret of our method also converges. Denote the regret bound of original Adam as $R(T)$, thus we have
\begin{equation}
	\lim _{T \rightarrow \infty} \frac{\tilde R(T)}{T} \le \lim _{T \rightarrow \infty} \frac{ R(T)+\Delta\tilde R(T)}{T}=0
\end{equation}
Therefore, our pruning method for momentum also achieves the following guarantee for all $T\ge1$:
\begin{equation}
	\frac{R(T)}{T}=O\left(\frac{1}{\sqrt{T}}\right)
\end{equation}
This indicates that our method also has good convergence rate as that of training without checkpoint compression. Detailed analysis can be found in Appendix~\ref{app:conv}.

It should be noted that Sashank~\etal~\cite{sashank2018convergence} point out the potential problem with the proof of convergence in~\cite{kingma2014adam} and tremendous efforts~\cite{shi2021rmsprop,zhang2022adam,defossez2020simple} have been taken for better analysis of convergence of Adam algorithm. In the above section, we provide an proof-of-concept analysis for the convergence of our checkpoint compression framework. Incorporating newer convergence analysis of Adam into our framework should be feasible since we only modify the moments in a specific iteration. Extensive experiments also demonstrate the convergence of our proposed method.

\subsection{Quantization}
\label{ssec:quant}
Besides pruning, quantization is also a common used method to compress the models or reduce the overhead of calculations. In our task, we only concern about the size of checkpoint, so we can choose non-uniform quantization method, which has better compression ratio but cannot accelerate the inference process.

Shown in Figure~\ref{fig:framework}, we quantize weights and momentum states separately. We leave the pruned weights or momentum states to zero, and apply K-means algorithm on other weights or momentum states to cluster them to $2^n-1$ cluster centers. Then we save the cluster centers $\mathcal{C}_t$ and cluster index $\mathcal{I}_t$.

\subsection{Compressing and Reconstructing Checkpoints}
\label{ssec:comp-recon}
Based on the methods we describe on the above sections, we here give a detailed introduction of our compressing and reconstructing process. During the whole training process, once we reach a saving node, we start our compression process independent of the main training process. We always keep a reconstructed version of the last checkpoint, for fast compression and training recovery. With this reconstructed version, the compressing process is described in Algorithm~\ref{alg:1}. 

\begin{figure}[t]
\vspace{-3mm}
\begin{algorithm}[H]
\caption{Compressing process}
\label{alg:1}
\begin{algorithmic} 
\REQUIRE last reconstructed weight checkpoint $\hat{\mathcal{W}}_{t-1}$, original weight checkpoint $\mathcal{W}_{t}$, original optimizer checkpoint $\mathcal{O}_t$
\STATE $ \Delta \mathcal{W}_t \leftarrow \mathcal{W}_t - \hat{\mathcal{W}}_{t-1}$
\STATE $\Delta \mathcal{W}^{*}_t, \mathcal{O}^*_t \leftarrow \rm joint\_prune(\Delta \mathcal{W}_t, \mathcal{O}_t)$
\STATE $\mathcal{I}^{\mathcal{W}}_t, \mathcal{C}^{\mathcal{W}}_t \leftarrow \rm quantize(\Delta \mathcal{W}^{*}_t)$
\STATE $\mathcal{I}^{\mathcal{O}}_t, \mathcal{C}^{\mathcal{O}}_t \leftarrow \rm quantize(\Delta \mathcal{O}^{*}_t)$
\STATE $\mathcal{P}'_t \leftarrow 7zip(\mathcal{I}^{\mathcal{W}}_t, \mathcal{C}^{\mathcal{W}}_t, \mathcal{I}^{\mathcal{O}}_t, \mathcal{C}^{\mathcal{O}}_t)$
\STATE $\rm save\ \mathcal{P}'_t$
\end{algorithmic}
\end{algorithm}
\vspace{-9mm}
\end{figure}

In which $\mathcal{I}$ and $\mathcal{C}$ represent the index and clustering center of non-uniform quantization, respectively. After we finish the compression of iteration $t$, we can either keep the original checkpoint  $\mathcal{W}_{t}$ or reconstruct the weight checkpoint $\hat{\mathcal{W}}_{t}$ for the compression of next checkpoint or fast resumption from training crashed. We delete the other checkpoints $\hat{\mathcal{W}}_{i}, i \le t-1$ to save the storage.

\begin{figure}[t]
	\vspace{-3mm}
\begin{algorithm}[H]
\caption{Reconstructing process}
\label{alg:2}
\begin{algorithmic} 
\REQUIRE last reconstructed weight checkpoint $\hat{\mathcal{W}}_{t-1}$, compressed checkpoint $\mathcal{P}_{t}$
\STATE $\mathcal{I}^{\mathcal{W}}_t, \mathcal{C}^{\mathcal{W}}_t, \mathcal{I}^{\mathcal{O}}_t, \mathcal{C}^{\mathcal{O}}_t \leftarrow \rm unzip(\mathcal{P}_t)$
\STATE $\Delta \mathcal{W}^{Q*}_t \leftarrow \mathcal{C}^{\mathcal{W}}_t[\mathcal{I}^{\mathcal{W}}_t]$
\STATE $\hat{\mathcal{O}}_t \leftarrow \mathcal{C}^{\mathcal{O}}_t[\mathcal{I}^{\mathcal{O}}_t]$
\STATE $\hat{\mathcal{W}}_{t} \leftarrow \hat{\mathcal{W}}_{t-1}/\mathcal{W}_{t-1} + \Delta \mathcal{W}^{Q*}_t$
\end{algorithmic}
\end{algorithm}
\vspace{-8mm}
\end{figure}

We can reconstruct a weight checkpoint $\hat{\mathcal{W}}_t$ from last saved/reconstructed weight checkpoint $\mathcal{W}_{t-1}$/$\hat{\mathcal{W}}_{t-1}$ and our saved compressed checkpoint $\mathcal{P}_t$ following the Algorithm~\ref{alg:2}.

\begin{figure}[t]
\vspace{-3mm}
\begin{algorithm}[H]
\caption{Reconstructing arbitrary checkpoints}
\label{alg:3}
\begin{algorithmic} 
\REQUIRE random seed $s$, compressed checkpoints $\mathcal{P}_{i}$, required iterations $t$
\STATE $\hat{\mathcal{W}}_0 = \rm init(s)$
\WHILE{$i < t$}
\STATE $\hat{\mathcal{W}}_i \leftarrow \rm recon(\hat{\mathcal{W}}_{i-1}, \mathcal{P}_{i})$
\ENDWHILE
\end{algorithmic}
\end{algorithm}
\vspace{-11mm}
\end{figure}

Once we finished the whole training process, only the random seed for initialize weights and the compressed checkpoints are required to be saved, which are significantly smaller than the whole weights and optimizer checkpoint. If we want to reconstruct arbitrary checkpoints, we can follow the Algorithm~\ref{alg:3}.

\section{Experiments}
\label{sec:exp}

\begin{table*}[t]
	\vspace{-10pt}
	\caption{Results of ViT-L32 model on ImageNet-1K dataset. CR represents the compression ratio. M2W pruning represents the Momentum-to-Weights pruning shown in equation~\ref{eqa:rw}. No check means that the weights are pruned with a setting threshold instead of $\frac{\alpha}{\sqrt{m_t}}$. And W2M pruning represents the Weights-to-Momentum pruning shown in equation~\ref{eqa:ro}. No check means that $\mathcal{M}_o(i) = \mathds{1}_{v_{t}(i) > r_o}$. * We estimate the results of the other work in terms of the checkpoint size of momentum being twice the weights.}
	\centering
	\begin{tabular}{@{}c|cc|ccc@{}}
		\toprule
		Method     & M2W pruning & W2M pruning & Top-1 Accuracy(\%) & CR(Weights)   & CR(Weights \& Momentum)  \\ \midrule
		baseline   &             &             & 71.36        & 1     & 1     \\
		CNR+       &             &             & 71.57        & 7.82  & 1.41*  \\
		QD+        &             &             & 71.24        & 16.31 & 1.45*  \\
		DynaQuant  &             &             & 71.82        & 26.19 & 1.47*  \\ \midrule
		ExCP(Ours) &             &             &      71.51        &  -     &  19.88     \\
		ExCP(Ours) &  \checkmark           &             &      71.53        &    -   &  25.54     \\
		ExCP(Ours) &             &  \checkmark          &       71.80       &  -     &    22.76   \\
		ExCP(Ours) &  \checkmark           &  \checkmark           & 71.69        & -     & 35.21 \\ \bottomrule
	\end{tabular}
	\label{table:vit}
\end{table*}

\begin{table*}[th]
	\vspace{-8pt}
	\caption{Results of Pythia-410M models. We achieve almost lossless results while the storage is reduce by $\sim$70$\times$.}
	\centering
	\begin{tabular}{c|c|c|ccccccc}
		\toprule
		\multirow{2}{*}{Model}       & \multirow{2}{*}{Method} & \multirow{2}{*}{Size} & \multicolumn{7}{c}{Tasks}                                   \\ \cmidrule(lr){4-10} 
		&                         &                       & hellaswag & arc-e &  piqa & C3 & csl & lambada & Avg \\ \midrule
		\multirow{3}{*}{Pythia-410M} & Original model          &         4.53G              &     \textbf{32.52}    &  35.80     &    \textbf{62.13}    &  \textbf{37.21}  &  \textbf{53.75}   &    37.22     &   43.11  \\
		& Residual+7Zip            &        3.40G           &    \textbf{32.52}    &  35.80     &   \textbf{62.13}    &  \textbf{37.21}  &  \textbf{53.75}   &    37.22    &  43.11  \\	
		& \textbf{ExCP (Ours)}                 &     \textbf{0.06G}                  &    31.95     &   \textbf{37.04}   & 62.62 &   36.22 & 52.50 & \textbf{37.24} &  42.93 \\ %
		\bottomrule
	\end{tabular}
	\label{table:res}
	\vspace{-3mm}
\end{table*}

\begin{table*}[t]
	
	\vspace{-10pt}
	\caption{Results of PanGu-$\pi$ series models. We achieve almost lossless results while the storage is reduced by $\sim25\times$.}
	\setlength{\tabcolsep}{2.15pt}
	\small
	\centering
	\begin{tabular}{@{}ccccccccccccccc@{}}
		\toprule
		\multicolumn{2}{c}{\multirow{2}{*}{Model}} & \multirow{2}{*}{Size} & \multicolumn{4}{c}{Examination}   & Knowledge & \multicolumn{2}{c}{Reasoning} & \multicolumn{4}{c}{Understanding} & \multirow{2}{*}{Avg} \\ \cmidrule(lr){4-14}
		\multicolumn{2}{c}{}                       &                       & C-Eval & CMMLU & MMLU  & AGI-Eval & BoolQ     & AX-b          & PIQA          & CSL    & EPRSTMT  & XSum  & LCSTS &                          \\ \midrule
		\multirow{2}{*}{PanGu-$\pi$-1B}      & Ori       & 14.64G                 & 38.05  & 37.86 & 34.96 & 30.42    & 58.62     & 43.75         & 62.02         & 56.25  & 55.62    & 16.00 & 14.60 & 40.74                    \\
		& Ours      & 0.59G                 & 36.71  & 38.65 & 37.13 & 31.87    & 59.30     & 42.66         & 61.10         & 55.00  & 56.25    & 16.31 & 14.14 & 40.83                    \\ \midrule
		\multirow{2}{*}{PanGu-$\pi$-7B}      & Ori       & 98.59G                & 59.91  & 59.97 & 61.84 & 54.04    & 64.59     & 56.88         & 77.31         & 63.12  & 90.00    & 19.59 & 16.61 & 56.71                    \\
		& Ours      & 4.10G                 & 61.32  & 60.14 & 62.37 & 55.11    & 68.44     & 52.90         & 77.91         & 63.75  & 90.00    & 19.24 & 16.77 & 57.09                    \\ \bottomrule
	\end{tabular}
	\label{table:pangu}
\end{table*}

\begin{figure*}[th]
	\centering
	\vspace{-5pt}
	\includegraphics[width=6.6in]{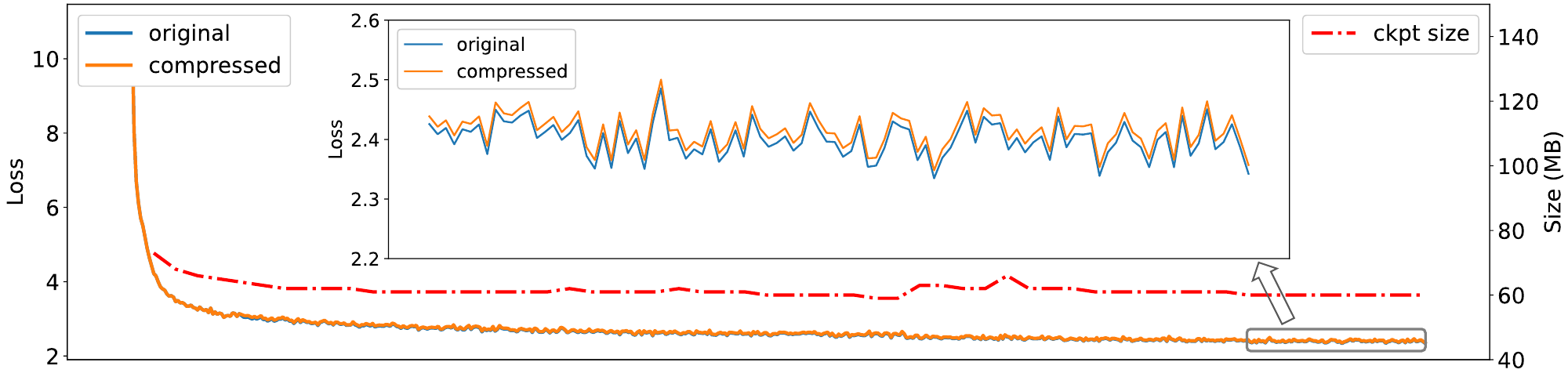}
	\vspace{-10pt}
	\caption{Comparison of training loss and checkpoint size between original models and our methods.}
	\vspace{-10pt}
	\label{fig:training}
\end{figure*}

\subsection{Models and Datasets}
\label{ssec:eval-data}
We conduct our experiments on ViT-L32~\cite{dosovitskiy2020image}, Pythia-410M~\cite{biderman2023pythia}, PanGu-$\pi$-1B and PanGu-$\pi$-7B~\cite{wang2023pangu} models. The ViT-L32 model is trained and evaluated on ImageNet-1K dataset. As for LLMs, we train Pythia-410M on on a subset of the standard Pile~\cite{gao2020pile} dataset. We used about 1/20 of the 300B tokens for this experiment. The PanGu-$\pi$ series models are trained following the training details of their paper~\cite{wang2023pangu}, which are trained on about 1.6 trillion tokens.

For evaluation, we use opencompass~\cite{2023opencompass} as the evaluation framework. We choose HellaSwag, ARC-easy, PIQA, C3, CSL and LAMBADA tasks to evaluate the performance of Pythia-410M, since these evaluation benchmarks are sensitive to model performance. And we evaluate the PanGu-$\pi$ series models following their paper~\cite{wang2023pangu} which uses 11 benchmarks belonging to 4 categories, examination, knowledge, reasoning and understanding.

\subsection{Implementation Details}
\label{ssec:impl-detail}
During the training process, we save checkpoints every epoch for ViT-L32 and every 1000 iterations for LLMs, and compress the checkpoints in the meantime. We break our training process periodically, and then we resume the training from our compressed checkpoints until we finish the whole training process.

Unless otherwise specified, we set the $\alpha$ in equation~\ref{eqa:rw} and $\beta$ in equation~\ref{eqa:ro} as $5e-5$ and $2.0$ in our experiments, respectively. The weights except zero are non-uniformly quantized to $2^n-1$ clustering center while the value zero occupies one center. And the bit number $n$ is set as 4 in experiments. We combine two int4 number into one int8 number while saving. 7zip compression algorithm is used for further storage reduction.

\subsection{Experimental Results}
\label{ssec:exp-res}

First we evaluate the effectiveness of our methods on ViT-L32 models to make a comparison with other checkpoint compression methods shown in Table~\ref{table:vit}. The accuracies of CNR+, QD+ and DynaQuant are all from Dynaquant~\cite{dynaquant}. We follow the same setting as Dynaquant that we break and resume the training process from compressed checkpoints every 15 epochs. The results show that our method achieve better compression ratio which reachs more than 30$\times$. Moreover, previous work do not compress the momentum of optimizer, which means that they achieve even less compression ratio while considering both weights and momentum. We also evaluate our joint pruning strategy. Using both M2W and W2M pruning achieve the best results.

Next we evaluate the performance of our methods on a relatively smaller LLM Pythia-410M. The results are shown in Table~\ref{table:res}. The period of breaking is set to 5000 iterations. From the results, we can find that our methods achieve nearly lossless compression with $\sim$70$\times$ compression ratio even the model itself is small enough in LLMs. Besides, our evaluation results on downstream tasks are even a little bit better than that of the original model. We think this is the improvement in generalization ability brought by the introduction of additional regularizers during the compression process. %

The training loss curve and the comparison of storage of each checkpoint during the training process is shown in Figure~\ref{fig:training}. We plot the training loss every 100 iterations. We can find that our training curve fits well with the original one, with only a small loss increase. Meanwhile, we achieve $\sim70\times$ compression in every node. %
When the learning rate is large, the stored checkpoint is relatively large. When the learning rate is small, the model size decreases since the model tends to be stable. This is related to our strategy of compressing residuals.

We also evaluate the PanGu-$\pi$-1B and PanGu-$\pi$-7B models in Table~\ref{table:pangu}. Since we reproduce their results, the original models may differ slightly from the results in their paper. Our compressed models achieve $\sim25\times$ storage reduce with negligible precision loss. Notice that the weights of PanGu-$\pi$ series models are originally saved in float16 format, so the compression ratio would be less than that of Pythia series models.

\begin{figure*}[th]
\begin{center}
	\begin{tcolorbox}[colback=gray!10,%
		colframe=black,%
		width=16cm,%
		arc=1mm, auto outer arc,
		boxrule=0.8pt,
		]
		\textbf{Question}: What kind of exercise do you suggest I do indoors?\\

\textbf{Original PanGu-$\pi$-1B}: If you are looking for a way to do some exercise in the comfort of your own home, there are several great exercises you can try. Some suggestions are:
Squats: This is a great exercise for all levels, from beginner to advanced. You can do them by sitting on a chair with a squat rack at the back. Slowly lower your hips until your thighs are parallel to the ground, then push it back up. 2. Plank: Planks are a great core-strengthening exercise that can help improve your posture and lower your heart rate. Start by getting into a pushup, then a plank, and finally a layover. 3. Mountain climbers: Mountain climbers are a fun and challenging exercise that can help improve your strength and endurance. Start by bringing a heavy object, like a bag or bottle, into a vertical position. Climb the object up to the ceiling, then bring it back down to the ground. 4. Mountain climbers: Mountain climbers are a challenging and exciting exercise that can help improve your upper body strength. Start by bringing a heavy object, like a bag or bottle, into a vertical position. Climb the object up to the ceiling, then bring it back down to the ground. 5. Yoga: Yoga is a great way to improve your flexibility, strength, and balance. There are many online videos and apps that you can use to learn. Remember, it's important to start slowly and gradually increase the intensity of your workouts. If you feel uncomfortable with any of these exercises, you can always turn back to doing them at your own pace.\\

\textbf{Compressed PanGu-$\pi$-1B}: There are a lot of great exercises you can do indoors to help improve your fitness, such as:
1. Yoga: Yoga is a great way to improve flexibility, strength, and balance. There are many online videos and classes available that can help you learn the basics. 2. Bodyweight exercises: You can do bodyweight exercises such as push-ups, squats, lunges, and planks to build strength and improve your overall fitness. 3. Dancing: Dancing is a fun way to improve your cardiovascular fitness, coordination, and balance. There are many online tutorials and classes available that can help you learn the basics. 4. Yoga for beginners: If you're new to yoga, there are many beginner classes available that can help you learn the basics. 5. Pilates: Pilates is a great way to improve your core strength, flexibility, and balance. There are many online videos and classes available that can help you learn the basics. Remember to start slow and gradually increase the intensity and duration of your workouts as you get stronger and more confident.

\end{tcolorbox}
\end{center}
\vspace{-12pt}
\caption{Q\&A example to show the difference between our compressed model and the original model.}
\label{fig:visualize}
\end{figure*}

We show a question \& answer results for PanGu-$\pi$-1B model in Figure~\ref{fig:visualize} to show the difference between our compressed model and the original model. In this example, our compressed model shows a better understanding to the limit \textit{indoors} and gives a better answer. It proves that our compressed model perfroms even better than the original one in some cases. More results are shown in the Appendix~\ref{sec:vis}.

\subsection{Ablation Studies}

\label{ssec:abl}
\begin{table}[t]
\vspace{-5pt}
\caption{Ablation study of our methods. Applying residual, joint-prune and quantization together achieves the best size while the average accuracy is almost lossless.}
\centering
\setlength{\tabcolsep}{9pt}
\begin{tabular}{ccccc}
\toprule
\multicolumn{3}{c}{method}      & \multirow{2}{*}{Size} & \multirow{2}{*}{Avg Acc} \\ \cmidrule(lr){1-3}
residual & prune & quant&                       &                          \\ \midrule
         &       &       &               4070M         &          43.11                \\
    \checkmark     &       &       &       3484M                 &         43.11                 \\
         &   \checkmark    &       &            324M            &             29.95             \\
         &      &    \checkmark    &            492M            &            40.17             \\
   \checkmark      &   \checkmark    &       &         276M              &     42.92                     \\
     \checkmark    &       &   \checkmark    &  493M & \textbf{42.94} \\
      \checkmark   &   \checkmark    &   \checkmark    &          \textbf{61M}             &       42.93                   \\ \bottomrule
\end{tabular}
\vspace{-15pt}
\label{table:abl}
\end{table}

We also do some ablation studies to show that every method in our compression pipeline is of vital importance. The results are shown in Table~\ref{table:abl}. Although calculating the residual of adjacent models cannot bring a significant storage reduce, it plays an important role in the whole pipeline. Directly pruning weights may harm the accuracy largely, while the residual of adjacent models does not have this problem. Joint-pruning and quantization on residual checkpoint separately compress the model by 15$\times$ and 8$\times$, respectively. Combining these two methods brings a large improvement to about 70$\times$.

\begin{table}[t]

\vspace{-5pt}
\caption{Ablation study of different quantization bins. We choose 4 bit in all other experiments since it achieves better performance-size trade-off.}
\vspace{2pt}
\centering
\setlength{\tabcolsep}{12pt}
\begin{tabular}{c|cc}
\toprule
Quant bins & Size & Avg Acc \\ \midrule
2 bit      &  \textbf{43M}    &  42.46       \\
4 bit      &  61M    &  \textbf{42.93}       \\
8 bit      &  87M    &   42.90      \\ \bottomrule
\end{tabular}
\vspace{-10pt}
\label{table:quant-bins}
\end{table}

We explore the influence of different quantization bins. From Table~\ref{table:quant-bins}, we can find that quantization below 4 bit cannot bring a significant storage reduce, so we choose 4 bit which achieves better performance-size trade-off. In some cases which extreme small checkpoint size is required, 2 bit could be used to further compress the checkpoints a little bit more.

\begin{table}[t]
\vspace{-5pt}
\caption{Comparison of different compression algorithms. We choose 7zip in all other experiments since it outperforms other algorithms.}
\centering
\setlength{\tabcolsep}{8pt}
\begin{tabular}{c|ccccc}
\toprule
     & zip  & rar  & rar4 & bz2  & 7z   \\ \midrule
size & 73M & 70M & 69M & 64M & \textbf{61M} \\ \bottomrule
\end{tabular}

\vspace{-15pt}
\label{table:compact}
\end{table}

We also evaluate different compression algorithms to compact the final checkpoint files. The results are shown in Table~\ref{table:compact}. The 7zip compression with LZMA2 algorithm achieves the best compression ratio, which leads to about 20\% less storage, and we apply 7zip with the ultra compression ratio on all other experiments.

\section{Conclusion}
\label{sec:conclusion}
In this paper, we discuss the extreme compression of LLM checkpoint. We propose a checkpoint compression framework which contains residual calculation, weights-momentum joint pruning and non-uniform quantization. We derive the criterion for weight-momentum joint-pruning and prove the convergence of the pruned momentum states. Experimental results show the effectiveness of our methods. We compress Pythia-410M by $\sim70\times$ while achieving nearly lossless results on down-stream evaluations.

In the future, we would try to extend the experiments to different tasks such as multi-modal large models and visual large models. And different types of neural networks including transformers, CNNs and RNNs would be taken into consideration.

\section*{Impact Statement}
This paper presents work whose goal is to advance the field of Machine Learning. There are many potential societal consequences of our work, none which we feel must be specifically highlighted here.
\bibliography{example_paper}
\bibliographystyle{icml2024}

\newpage
\appendix
\onecolumn

\section{Convergence Analysis.}
\label{app:conv}
\begin{theorem}\label{theorem}
According the convergence analysis in Adam~\cite{kingma2014adam}, assume that the function $f_t$ has bounded gradients, $\left\|\nabla f_t(\theta)\right\|_2 \leq G,\left\|\nabla f_t(\theta)\right\|_{\infty} \leq$ $G_{\infty}$ for all $\theta \in R^d$ and distance between any $\theta_t$ generated by Adam is bounded, $\left\|\theta_n-\theta_m\right\|_2 \leq D$, $\left\|\theta_m-\theta_n\right\|_{\infty} \leq D_{\infty}$ for any $m, n \in\{1, \ldots, T\}$, and $\beta_1, \beta_2 \in[0,1)$ satisfy $\frac{\beta_1^2}{\sqrt{\beta_2}}<1$. Let $\alpha_t=\frac{\alpha}{\sqrt{t}}$ and $\beta_{1, t}=\beta_1 \lambda^{t-1}, \lambda \in(0,1)$. If we prune the moments with a mask $\mathcal{M}_o\in \{0,1\}$ at iteration $\tau$, 
Adam could also achieves the following guarantee, for all $T \geq 1$.
\begin{equation}
	\small
	\begin{aligned}\label{eq2}
		\tilde R(T) 
		\leq & \frac{D^2}{2 \alpha\left(1-\beta_1\right)} \sum_{i=1}^d \sqrt{T \widehat{v}_{T, i}}
		+\frac{\alpha\left(1+\beta_1\right) G_{\infty}}{\left(1-\beta_1\right) \sqrt{1-\beta_2}(1-\gamma)^2} \sum_{i=1}^d\left\|g_{1: T, i}\right\|_2 
		+\frac{D_{\infty}^2 G_{\infty} \sqrt{1-\beta_2}}{2 \alpha} \sum_{i=1}^d \sum_{t=1}^t \frac{\beta_{1, t}}{\left(1-\beta_{1, t}\right)} \sqrt{t}\\
		& +\frac{D^2}{2 \alpha\left(1-\beta_1\right)} \sum_{i=1}^d (\sqrt{T \widehat{v}_{\tau, i}} - \sqrt{T \widehat{v}_{\tau, i}\mathcal{M}_o})
	\end{aligned}
\end{equation}
\end{theorem}
\begin{proof}
According the convergence analysis in Adam~\cite{kingma2014adam}, assume that the function $f_t$ has bounded gradients, $\left\|\nabla f_t(\theta)\right\|_2 \leq G,\left\|\nabla f_t(\theta)\right\|_{\infty} \leq$ $G_{\infty}$ for all $\theta \in R^d$ and distance between any $\theta_t$ generated by Adam is bounded, $\left\|\theta_n-\theta_m\right\|_2 \leq D$, $\left\|\theta_m-\theta_n\right\|_{\infty} \leq D_{\infty}$ for any $m, n \in\{1, \ldots, T\}$, and $\beta_1, \beta_2 \in[0,1)$ satisfy $\frac{\beta_1^2}{\sqrt{\beta_2}}<1$. Let $\alpha_t=\frac{\alpha}{\sqrt{t}}$ and $\beta_{1, t}=\beta_1 \lambda^{t-1}, \lambda \in(0,1)$. Adam achieves the following guarantee, for all $T \geq 1$.
\begin{equation}
	R(T) \leq \frac{D^2}{2 \alpha\left(1-\beta_1\right)} \sum_{i=1}^d \sqrt{T \widehat{v}_{T, i}}+\frac{\alpha\left(1+\beta_1\right) G_{\infty}}{\left(1-\beta_1\right) \sqrt{1-\beta_2}(1-\gamma)^2} \sum_{i=1}^d\left\|g_{1: T, i}\right\|_2+\sum_{i=1}^d \frac{D_{\infty}^2 G_{\infty} \sqrt{1-\beta_2}}{2 \alpha \beta_1(1-\lambda)^2},
\end{equation}
where $R(T)$ is the regret:
\begin{equation}\label{eq:regret}
	R(T) = \sum_{t=1}^{T} [f_t(\theta_t)-f_t(\theta^*)]
\end{equation}
This theorem could be obtained by the following:
\begin{equation}
	\begin{aligned}\label{eq1}
		R(T) \leq & \sum_{i=1}^d \frac{1}{2 \alpha_1\left(1-\beta_1\right)}\left(\theta_{1, i}-\theta_{, i}^*\right)^2 \sqrt{\widehat{v}_{1, i}}+\sum_{i=1}^d \sum_{t=2}^T \frac{1}{2\left(1-\beta_1\right)}\left(\theta_{t, i}-\theta_{, i}^*\right)^2\left(\frac{\sqrt{\widehat{v}_{t, i}}}{\alpha_t}-\frac{\sqrt{\widehat{v}_{t-1, i}}}{\alpha_{t-1}}\right) \\
		& +\frac{\beta_1 \alpha G_{\infty}}{\left(1-\beta_1\right) \sqrt{1-\beta_2}(1-\gamma)^2} \sum_{i=1}^d\left\|g_{1: T, i}\right\|_2+\frac{\alpha G_{\infty}}{\left(1-\beta_1\right) \sqrt{1-\beta_2}(1-\gamma)^2} \sum_{i=1}^d\left\|g_{1: T, i}\right\|_2 \\
		& +\sum_{i=1}^d \sum_{t=1}^T \frac{\beta_{1, t}}{2 \alpha_t\left(1-\beta_{1, t}\right)}\left(\theta_{, i}^*-\theta_{t, i}\right)^2 \sqrt{\widehat{v}_{t, i}}
	\end{aligned}
\end{equation}

\begin{equation}
	\begin{aligned}\label{eq2}
		R(T) \leq & \frac{D^2}{2 \alpha\left(1-\beta_1\right)} \sum_{i=1}^d \sqrt{T \widehat{v}_{T, i}}+\frac{\alpha\left(1+\beta_1\right) G_{\infty}}{\left(1-\beta_1\right) \sqrt{1-\beta_2}(1-\gamma)^2} \sum_{i=1}^d\left\|g_{1: T, i}\right\|_2+\frac{D_{\infty}^2}{2 \alpha} \sum_{i=1}^d \sum_{t=1}^t \frac{\beta_{1, t}}{\left(1-\beta_{1, t}\right)} \sqrt{t \widehat{v}_{t, i}} \\
		\leq & \frac{D^2}{2 \alpha\left(1-\beta_1\right)} \sum_{i=1}^d \sqrt{T \widehat{v}_{T, i}}+\frac{\alpha\left(1+\beta_1\right) G_{\infty}}{\left(1-\beta_1\right) \sqrt{1-\beta_2}(1-\gamma)^2} \sum_{i=1}^d\left\|g_{1: T, i}\right\|_2 \\
		& +\frac{D_{\infty}^2 G_{\infty} \sqrt{1-\beta_2}}{2 \alpha} \sum_{i=1}^d \sum_{t=1}^t \frac{\beta_{1, t}}{\left(1-\beta_{1, t}\right)} \sqrt{t}
	\end{aligned}
\end{equation}

In our method, we prune some variables for the momentum, \ie, a mask $\mathcal{M}_o$ is applied for $m$ and $v$. Assume that we prune the momentum at iteration $\tau$, the convergence is the same as original optimization process for iteration $1$ to iteration $\tau$. However, the first and second moment vectors $v_{\tau}$ and $m_{\tau}$ become $v_{\tau}\mathcal{M}_o$ and $m_{\tau}\mathcal{M}_o$ at iteration $\tau$.

From Eq.~\ref{eq1}, we have:
\begin{equation}
	\begin{aligned}
		\tilde R(T) \leq & \sum_{i=1}^d \frac{1}{2 \alpha_1\left(1-\beta_1\right)}\left(\theta_{1, i}-\theta_{, i}^*\right)^2 \sqrt{\widehat{v}_{1, i}}+\sum_{i=1}^d \sum_{t=2}^{\tau-1} \frac{1}{2\left(1-\beta_1\right)}\left(\theta_{t, i}-\theta_{, i}^*\right)^2\left(\frac{\sqrt{\widehat{v}_{t, i}}}{\alpha_t}-\frac{\sqrt{\widehat{v}_{t-1, i}}}{\alpha_{t-1}}\right) \\
		&\sum_{i=1}^d \frac{1}{2\left(1-\beta_1\right)}\left(\theta_{\tau, i}-\theta_{, i}^*\right)^2\left(\frac{\sqrt{\mathcal{M}_o\widehat{v}_{\tau, i}}}{\alpha_\tau}-\frac{\sqrt{\widehat{v}_{\tau-1, i}}}{\alpha_{\tau-1}}\right)+\sum_{i=1}^d \sum_{t={\tau+1}}^{T} \frac{1}{2\left(1-\beta_1\right)}\left(\theta_{t, i}-\theta_{, i}^*\right)^2\left(\frac{\sqrt{\widehat{v}_{t, i}}}{\alpha_t}-\frac{\sqrt{\widehat{v}_{t-1, i}}}{\alpha_{t-1}}\right) \\
		& +\frac{\beta_1 \alpha G_{\infty}}{\left(1-\beta_1\right) \sqrt{1-\beta_2}(1-\gamma)^2} \sum_{i=1}^d\left\|g_{1: T, i}\right\|_2+\frac{\alpha G_{\infty}}{\left(1-\beta_1\right) \sqrt{1-\beta_2}(1-\gamma)^2} \sum_{i=1}^d\left\|g_{1: T, i}\right\|_2 \\
		& +\sum_{i=1}^d \sum_{t=1}^T \frac{\beta_{1, t}}{2 \alpha_t\left(1-\beta_{1, t}\right)}\left(\theta_{, i}^*-\theta_{t, i}\right)^2 \sqrt{\widehat{v}_{t, i}}
	\end{aligned}
\end{equation}
Under similar assumption as Eq.~\ref{eq2}, we could have the following regret bound:
\begin{equation}
	\begin{aligned}\label{eq2}
		\tilde R(T) 
		\leq & \frac{D^2}{2 \alpha\left(1-\beta_1\right)} \sum_{i=1}^d \sqrt{T \widehat{v}_{T, i}}+\frac{\alpha\left(1+\beta_1\right) G_{\infty}}{\left(1-\beta_1\right) \sqrt{1-\beta_2}(1-\gamma)^2} \sum_{i=1}^d\left\|g_{1: T, i}\right\|_2 
		+\frac{D_{\infty}^2 G_{\infty} \sqrt{1-\beta_2}}{2 \alpha} \sum_{i=1}^d \sum_{t=1}^t \frac{\beta_{1, t}}{\left(1-\beta_{1, t}\right)} \sqrt{t}\\
		& +\frac{D^2}{2 \alpha\left(1-\beta_1\right)} \sum_{i=1}^d (\sqrt{T \widehat{v}_{\tau, i}} - \sqrt{T \widehat{v}_{\tau, i}\mathcal{M}_o})
	\end{aligned}
\end{equation}
\end{proof}

\section{More Visualization of Compressed LLM.}
\label{sec:vis}
Here we show more visualization results in Figure~\ref{fig:visualize2} and Figure~\ref{fig:visualize3}.

\begin{figure*}[h]
	\begin{center}
		\begin{tcolorbox}[colback=gray!10,%
			colframe=black,%
			width=16cm,%
			arc=1mm, auto outer arc,
			boxrule=0.8pt,
			]
			\textbf{Question}: What is the area of a circle with a radius of two meters?\\
			
			\textbf{Original PanGu-$\pi$-1B}: The formula for the area of a circle is A = $\pi r^2$, where r is the radius.\\
			Substituting r = 2 meters, we get:\\
			$A = \pi(2)^2$\\
			$A = 4\pi$ square meters\\
			Therefore, the area of a circle with a radius of 2 meters is approximately 4$\pi$ square meters or approximately 12.57 square meters.\\
			
			\textbf{Compressed PanGu-$\pi$-1B}: The formula for the area of a circle is $A = \pi r^2$, where r is the radius.
			To find the area of a circle with a radius of two meters, we can use the formula:\\
			$A = \pi(2)^2$\\
			$A = \pi(4)$\\
			$A \approx 12.57$ square meters\\
			So, the area of a circle with a radius of two meters is approximately 12.57 square meters. \\
			
			\textbf{Question}: What are the seven continents on Earth?\\
			
			\textbf{Original PanGu-$\pi$-1B}: The seven continents are: 1. Asia 2. Africa 3. North America 4. South America 5. Antarctica 6. Europe 7. Australia/Oceania\\
			
			\textbf{Compressed PanGu-$\pi$-1B}: The seven continents are: 1. Asia 2. Africa 3. North America 4. South America 5. Antarctica 6. Europe 7. Australia
		\end{tcolorbox}
	\end{center}
	\vspace{-12pt}
	\caption{Q\&A example to show the difference between our compressed model and the original model.}
	\label{fig:visualize2}
\end{figure*}

\begin{figure*}[t]
	\begin{center}
		\begin{tcolorbox}[colback=gray!10,%
			colframe=black,%
			width=16cm,%
			arc=1mm, auto outer arc,
			boxrule=0.8pt,
			]
			\textbf{Question}: How to find a suitable picnic spot?\\
			\textbf{Original PanGu-$\pi$-1B}: Finding a picnic spot can be a great way to unwind and enjoy a meal with friends or family. Here are a few tips to help you find the perfect picnic spot:
1. Look for an outdoor space: Whether it's a park, a beach, or a forest, finding an outdoor space to set out a picnic can be a great way to enjoy the outdoors.
2. Check for shady areas: A picnic should be a fun and enjoyable experience, so make sure to check for shady areas where people might be hiding or eating.
3. Look for a busy area: If you're picnicking during the day, look for an area with a lot of people gathered, like a park or mall park.
4. Check for views: A great picnic spot should have a great view, so think of a location that overlooks a large body of water or a park entrance.
5. Look for a comfortable spot: A picnic should be a relaxing experience, so think of a location that is comfortable and easy to lie down or sleep in.
6. Take into account seasonal weather: If you're picnicking in the spring or fall, look for a location that is shaded during those months, as those times can be pleasant.
7. Ask locals: If you're trying to find a picnic spot, don't be afraid to ask locals. They may have some of the best spots for picnicking and enjoying a good picnic.\\
			\textbf{Compressed PanGu-$\pi$-1B}: There are several ways to find a suitable picnic spot:
1. Ask for recommendations: Ask family and friends for recommendations on a suitable spot for the picnic. They may know a place that they have used or that they think is a good choice.
2. Check online: Check online for suggestions on picnic spots in your area. Many websites have a section for finding suitable spots for picnics, or you can search for them by using search engines.
3. Look for public spaces: Look for parks, gardens, or other public spaces that are often free of trees and open spaces for picnics.
4. Check with local authorities: Contact your local government to see if they have any upcoming events or festivals that require a picnic. They may be able to provide you with some recommendations on a suitable spot.
5. Ask for a description: Ask the host or hostess at the restaurant to provide a description of the picnic spot, including the number of people, the area, and the amenities available.
Remember to take necessary precautions to ensure the safety of yourself and others while enjoying a picnic.\\
			\textbf{Question}: How to take beautiful landscape photos?\\
			\textbf{Original PanGu-$\pi$-1B}: 1. Understand your subject: landscape photos generally require some understanding of the subject in order to accurately capture its beauty. Study the subject, including its culture, history, and natural environment, to be able to anticipate its behavior and lighting conditions.
			2. Use the right lighting: Landscape photography often requires the use of artificial lighting to illuminate the subject. Experiment with different lighting techniques, such as natural lighting, studio lighting, or table lamp lighting, to create the desired effect.
			3. Adjust camera settings: Most cameras have settings that can either adjust depth of field, aperture, and ISO. Experiment with these settings to find the right balance of light and depth, and to get the desired exposure ratio, depth, and color.
			4. Take multiple shots: Taking multiple photos from different angles can help to create a more dynamic and visually interesting landscape. Take advantage of vantage points, such as bridges, buildings, and mountains, to increase the subject's exposure and clarity.
			5. Add foreground elements: Adding a foreground element, such as trees, mountains, or buildings, can help to add depth and context to your photo. Use the rule of thirds to create a pleasing balance between the foreground and background.
			6. Edit your photos: Once you've captured your landscape photos, take the time to edit your photos. You can adjust brightness, contrast, and saturation, and add filters or editing tools to enhance the subject.
			7. Practice, practice, practice: The more you practice taking landscape photos, the better you'll get at capturing the beauty of the landscape. Keep experimenting with different lighting, camera settings, and editing techniques to keep your photos looking their best.\\
			\textbf{Compressed PanGu-$\pi$-1B}: 1. Understand your environment: Be sure to study your chosen location, the time and weather, and the type of subject you are photographing.
			2. Choose the right camera: Not all cameras are created equal, so take the time to research and choose the camera that best suits your needs.
			3. Get the right lighting: The lighting can dramatically affect the mood and tone of your photos. Try to capture the light sources that are most effective for your subject.
			4. Experiment with composition: Composition is an essential part of photography, so experiment with different angles, vantage points, and compositions to see what works best for your subject.
			5. Edit your photos: Once you have your compelling landscape photos, it's time to edit them. You can use software to add filters, adjust brightness and contrast, and enhance the overall look of your photos.
			6. Share your photos: Take a camera and post your photos on social media platforms like Instagram, Facebook, and YouTube. This will allow others to view and appreciate your work, and also help you gain more followers.
		\end{tcolorbox}
	\end{center}
	\vspace{-12pt}
	\caption{Q\&A example to show the difference between our compressed model and the original model.}
	\label{fig:visualize3}
\end{figure*}

\end{CJK}
\end{document}